\documentclass{article}
\usepackage{ijcai20}

\usepackage{times}
\usepackage{soul}
\usepackage{url}
\usepackage[hidelinks]{hyperref}
\usepackage[utf8]{inputenc}
\usepackage[small]{caption}
\usepackage{graphicx}
\usepackage{amsmath}
\usepackage{amsthm}
\usepackage{booktabs}
\usepackage{algorithm}
\usepackage{algorithmic}

\usepackage{amsfonts}
\usepackage{multirow}
\newtheorem{theorem}{Theorem}
\newtheorem{corollary}{Corollary}

\newtheorem{property}{Property}
\newtheorem{definition}{Definition}

\DeclareMathOperator*{\concat}{\scalebox{1}[1.0]{$\parallel$}}

\title{Learning Universal and Robust 3D Molecular Representations with Graph Convolutional Networks}

\author{
Shuo Zhang$^{1,2}$\and
Yang Liu$^{2}$\and
Li Xie$^{2}$\and
Lei Xie$^{1,2,3}$\\
\vspace{5pt}
\affiliations
\large{$^1$Ph.D. Program in Computer Science, The Graduate Center, The City University of New York\\
$^2$Department of Computer Science, Hunter College, The City University of New York\\
$^3$Helen \& Robert Appel Alzheimer’s Disease Research Institute, Feil Family Brain \& Mind Research Institute, Weill Cornell Medicine, Cornell University}
\emails
\large{szhang4@gradcenter.cuny.edu,
lei.xie@hunter.cuny.edu\\
\vspace{5pt}
January 21st, 2020}
}

\begin{document}

\maketitle

\begin{abstract}
To learn accurate representations of molecules, it is essential to consider both chemical and geometric features. To encode geometric information, many descriptors have been proposed in constrained circumstances for specific types of molecules and do not have the properties to be ``robust": 1. Invariant to rotations and translations; 2. Injective when embedding molecular structures. In this work, we propose a universal and robust Directional Node Pair (DNP) descriptor based on the graph-representations of 3D molecules. Our DNP descriptor is robust compared to previous ones and can be applied to multiple molecular types. To combine the DNP descriptor and chemical features in molecules, we construct the Robust Molecular Graph Convolutional Network (RoM-GCN) which is capable to take both node and edge features into consideration when generating molecule representations. We evaluate our model on protein and small molecule datasets. Our results validate the superiority of the DNP descriptor in incorporating 3D geometric information of molecules. RoM-GCN outperforms all compared baselines.
\end{abstract}

\section{Introduction}\label{Intro}
Machine learning, especially deep learning, has been widely used in molecule-related tasks. To this end, it is essential to convert chemical compounds or biological molecules, e.g. proteins, to representations in an embedding space that can be used as inputs for the deep learning models. 

Among various molecular representations, the 1D or 2D-based descriptors~\cite{weininger1988smiles,zhou2016cnnsite} are popular. However, their limitations are obvious: All these descriptors lack stereoscopic information, which is essential for a rich representation of molecular information and tasks related to 3D conformations of molecules. Thus, over the past few years, a growing body of work has focused on obtaining the representations for 3D molecules.

The 3D-based representations can be obtained by voxelizing molecules into 3D grids and learned with 3D convolutional neural networks (3D CNNs)~\cite{townshend2019end}. However, 3D grids are not invariant to translations and rotations, which are important for the prediction of molecular properties that naturally hold the invariances. Moreover, because of the additional dimensions, voxel-based representations are more computationally expensive in both generating and training compared with the 1D and 2D cases. 

Another way to represent the 3D molecular structures is to consider molecules as connectivity graphs. In several approaches, the 3D geometric information is incorporated into the graphs by using 3D coordinates as node features~\cite{qi2017pointnet,wang2019dynamic}. However, the coordinate-based approaches have the same drawbacks as the voxel-based representations: they do not capture geometric invariances. One way to solve this issue is to use the distance or angle information between the nodes, which captures the invariant relationships in 3D molecular structures. To incorporate those descriptors, graph convolutional networks (GCNs) are used with different message-passing schemes~\cite{gilmer2017neural,schutt2017schnet,fout2017protein,chen2019graph}. However, GCNs usually focus on local geometric relations and only contain distance information within a cutoff. So that the representational powers of those GCN-based models are limited and the resulting representations fail to map molecular structures injectively~\cite{klicperadirectional}. Thus, designing and applying more robust geometric descriptors on molecular graphs are highly expected. Another issue is that most GCN-based models work on the molecular graphs in which the nodes do not have intrinsic geometric information such as directional vectors. The advantage of introducing directional vectors is that they can bring rich geometric information (e.g. orientations) to enhance the representation.

In this work, for any type of molecules, we propose a novel 2-step method to learn a robust representation of 3D structures that contains general geometric information. We denote the term ``robust" as: {\it 1. Invariant to rotations and translations; 2. Injective when embedding molecular structures.} In detail, our 2-step method includes two components: First, the Directional Node Pair (DNP) descriptor, which is a robust geometric descriptor for graphs with directional vector associated nodes. Second, the Robust Molecular Graph Convolutional Network (RoM-GCN), which has novel aggregation schemes to update both of the node and edge features. With our DNP descriptor as edge features, the RoM-GCN can precisely learn the robust 3D molecular representations. The main contributions of our work are as follows:
\begin{itemize}
\item We propose a novel and robust geometric descriptor for 3D molecular structures called Directional Node Pair (DNP) descriptor. The DNP descriptor is built on molecular graphs while the nodes are associated with directional vectors.
\item We propose the Robust Molecular Graph Convolutional Network (RoM-GCN), which can update the node and edge features in a graph together for a higher representational power. The RoM-GCN can obtain robust representations for 3D molecules with the DNP descriptor.
\item We conduct experiments on protein and small molecule datasets and validate the superiority of the DNP descriptor as well as RoM-GCN. Our method outperforms all compared baselines.
\end{itemize}

\section{Related Work}
\paragraph{Graph Convolutional Networks.}
To learn the representations of graph-structured data, Graph Convolutional Networks (GCNs) have been proposed~\cite{duvenaud2015convolutional,niepert2016learning,kipf2017semi}. To improve the representational power, researchers are actively investigating different aspects of GCNs: \cite{xu2018how,morris2019weisfeiler} have built powerful GCNs inspired by comparing GCNs with Weisfeiler-Lehman test of isomorphism. \cite{gilmer2017neural,chen2019graph} have designed GCNs that can learn from both node and edge features. The GCN that we propose focuses on this aspect and utilizes a novel way to incorporate edge features.

\paragraph{Representations for molecules.} Many efforts have been made to obtain molecular representations. Early approaches for such embeddings include fingerprints like the SMILES~\cite{weininger1988smiles}. Over the past few years, people have started to use voxelized 3D grids together with 3D CNN for various tasks~\cite{ragoza2017protein,derevyanko2018deep,townshend2019end}. Other approaches for 3D molecular structures representation rely on graphs. GCNs can be applied on those graph representations to get embeddings of 3D molecules~\cite{duvenaud2015convolutional,kearnes2016molecular,schutt2017schnet,fout2017protein,klicperadirectional}. However, the descriptors used in those works for the embedding of 3D structures are neither robust in terms of transformation invariant and injective nor generalized for any types of molecules. A universal representation of 3D molecules is important since many biological problems such as protein-ligand interaction involve multiple types of molecules. Therefore, we propose a universal and robust geometric descriptor in this work.

\section{Robust Geometric Descriptor for 3D Molecular Structures} \label{Section3}

\begin{table}[tbp]
	\begin{center}
		\resizebox{1.\linewidth}{!}{
		    \normalsize
			\begin{tabular}{lcc}
				\toprule
				{3D Geometric Descriptor} & {Property \ref{Property1}} & {Property \ref{Property2}}  \\
				\midrule
				Voxel-based~\cite{townshend2019end} & $\times$ & \checkmark\\
				Coordinate-based~\cite{wang2019dynamic} & $\times$ & \checkmark\\
				Distance-based~\cite{chen2019graph} & \checkmark & $\times$\\
				Distance and angle-based~\cite{klicperadirectional} & \checkmark & $\times$\\
				Point Feature Histogram~\cite{rusu2009fast} & \checkmark & $\times$\\
				Point Pair Feature~\cite{drost2010model} & \checkmark & $\times$\\
				{\bf Directional Node Pair (ours)} & {\bf \checkmark} & {\bf \checkmark}\\
				\bottomrule
			\end{tabular}	
		}
	\end{center}
	\vspace{-5pt}
	\caption{Summary of whether or not a geometric descriptor for 3D structures has the properties we proposed in Section \ref{Properties} to define a robust geometric descriptor. ``$\checkmark$'' denotes that the descriptor holds the property, ``$\times$'' denotes that the descriptor doesn't hold the property. For the geometric descriptors in GCN-based methods, we evaluate the Property \ref{Property2} on the cases where the graph-representations contain our defined directional nodes in Section \ref{Directional_Node}.}\label{Summary_property}
	\vspace{-5pt}
	\label{tab:rules}
\end{table}

To solve tasks related to molecules in 3D Euclidean space, a necessary step is to get molecular representations that contain chemical and geometric information. In this section, we will focus on the geometric information. Our goal is to propose a ``robust" geometric descriptor that can be used in ML models for further representations of 3D molecules. In Section \ref{Properties}, we will introduce two properties to define a robust geometric descriptor. In Table \ref{Summary_property}, we summarize the geometric descriptors discussed in Section \ref{Intro} and Section \ref{Robust-Descriptor}. We find all those descriptors are not robust under our definition. Then in Section \ref{Robust-Descriptor}, we propose the robust Directional Node Pair (DNP) descriptor based on graph representations described in Section \ref{Directional_Node}. In Section \ref{RoM-GCN}, the DNP descriptor is used to embed the geometric information and to feed the RoM-GCN model to generate robust molecular representations.

\subsection{Properties of Robust Geometric Descriptor}\label{Properties}
Here we use two properties to define a ``robust" descriptor:

\begin{property} \label{Property1}
A ``robust" geometric descriptor $D$ of a molecular structure $s \in \mathbb{R}^3$ has the rotational and translational invariance: For all $s$, D(s) = f (D(s)), where $f: \mathbb{R}^3 \rightarrow \mathbb{R}^3$ is a transformation function for rotation or translation.
\end{property}

\begin{property} \label{Property2}
A ``robust" geometric descriptor $D$ of molecular structures $S \in \mathbb{R}^3$ is an \textbf{injective} mapping: For all $s_1, s_2 \in S$, whenever $D(s_1) = D(s_2)$, then $s_1 = s_2$.
\end{property}

To obtain geometric descriptor that holds both Property \ref{Property1} and \ref{Property2}, recent GCN-based work~\cite{klicperadirectional} makes progress by modeling the directions between pairs of atoms besides distances. The directional information is used by leveraging the angles between the directions in the aggregations of GCN. Such geometric descriptor has higher representational power than the distance-based geometric descriptor. However, the directional information in such descriptor is defined only between the neighboring nodes. The lack of directional feature for each node limits the application of the descriptor to the cases which we will discuss in the next section.

\subsection{Graphs with Directional Nodes} \label{Directional_Node}
We now introduce the graph representation in which the nodes are associated with directional vectors. We denote the nodes with directional vectors as ``directional nodes":

\begin{definition} [Directional Node]\label{Definition}
Let $G=(V, E)$ denote a graph, a node $v \in V$ is a directional node if it is associated with a directional vector $\boldsymbol{u} \in \mathbb{R}^3$.
\end{definition}

The advantage of directional nodes is the geometric information (e.g. orientations) contained in the directional vector feature is indispensable for the graphs with Euclidean structure, e.g. proteins, nucleotides, small molecules, etc. If we use the descriptors that only contain the geometric features {\it between} nodes, the intrinsic geometric features of the nodes will be lost. As a result, such descriptors do not meet the Property \ref{Property2}. For example, in the atom-level graph representation of 3D protein structures, we can treat the side chain R in each amino acid as a node. Since the R group is free to rotate about the $C_\alpha-R$ bond, proteins with the same graph structure but different R group rotamers will fail to be distinguished by those geometric descriptors~\cite{schutt2017schnet,chen2019graph,klicperadirectional} that are designed for graphs without directional nodes.

\begin{figure}[t]
	\centering
	\includegraphics[scale=0.7]{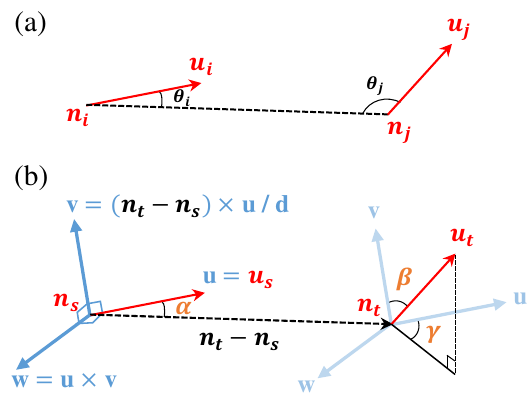}
	\vspace{-5pt}
	\caption{(a) Illustration of two directional nodes $\boldsymbol{n}_i$ and $\boldsymbol{n}_j$ with $\boldsymbol{u}_i$ and $\boldsymbol{u}_j$. (b) Illustration of the computations in Equation (\ref{PFH}).} 
	\vspace{-5pt}
	\label{DNP_figure}
\end{figure}

\subsection{Directional Node Pair (DNP) Descriptor}\label{Robust-Descriptor}
In this section, we propose the Directional Node Pair (DNP) descriptor, which is a robust geometric descriptor designed for 3D molecular structures that is represented by graphs that contain directional nodes. To meet Property \ref{Property1}, we use the distance and angle information between nodes that are invariant to rotations and translations. To meet Property \ref{Property2}, our descriptor should at first be an injective mapping of the geometric relations between a pair of directional nodes. Then an injective mapping of the geometric relations between all nodes can be obtained. 

The idea of DNP descriptor is inspired by the Point Feature Histograms (PFH) descriptor~\cite{rusu2009fast}, which defines a fixed coordinate frame at one of the nodes to indicate the relative geometric information between two nodes. The relationship is expressed using a quadruplet $\left \langle \alpha , \beta , \gamma , d \right \rangle$, where $\alpha , \beta , \gamma$ are 3 angles in the measurement and $d$ is the Euclidean distance between nodes. 

In detail, given two different nodes $\boldsymbol{n}_i=(x_i, y_i, z_i)$ and  $\boldsymbol{n}_j=(x_j, y_j, z_j), i \neq j$ and two direction vectors $\boldsymbol{u}_i = (u_{xi}, u_{yi}, u_{zi})$ and $\boldsymbol{u}_j = (u_{xj}, u_{yj}, u_{zj})$, we have $d={\left \| \boldsymbol{n}_i - \boldsymbol{n}_j  \right \|}_2$. To uniquely define a fixed coordinate frame at one of the two nodes, we first compute $\theta_i = \arccos{(\boldsymbol{u}_i \cdot \boldsymbol{v}_{ji})}$ and $\theta_j = \arccos{(\boldsymbol{u}_j \cdot \boldsymbol{v}_{ij})}$, where $\boldsymbol{v}_{ji} = \boldsymbol{n}_j - \boldsymbol{n}_i$, $\boldsymbol{v}_{ij} = \boldsymbol{n}_i - \boldsymbol{n}_j$. The situation is illustrated in Figure \ref{DNP_figure}(a). Then the three angles $\alpha , \beta , \gamma$ are obtained in the cases below:

(1) In corner cases:
\begin{align}
\small
(\alpha , \beta , \gamma) = 
\begin{cases}
(0, 0, 0),&\text{if }\theta_i \text{ and } \theta_j \text{ do not exist}\\
(\theta_k, 0, 0),&\text{if only }\theta_k \text{ exist}, k = i \text{ or } j\\
(0, \pi / 2, \pi),&\text{if }\theta_i = \theta_j = 0\\ 
(\pi, \pi / 2, \pi),&\text{if }\theta_i = \theta_j = \pi\\
(0, \pi / 2, 0),&\text{if }\theta_i = 0, \theta_j = \pi \text{ or } \theta_i = \pi, \theta_j = 0
\end{cases}.
\nonumber
\end{align}

(2) Otherwise if $\theta_i, \theta_j \neq 0$ and $\pi$, we use the following rules to define the source node $\boldsymbol{n}_s$ and the target node $\boldsymbol{n}_t$:
\begin{align}
\small
\text{if }\theta_i \leq \theta_j,
\left\{\begin{matrix}
\boldsymbol{n}_s = \boldsymbol{n}_i, \boldsymbol{u}_s = \boldsymbol{u}_i\\ 
\boldsymbol{n}_t = \boldsymbol{n}_j, \boldsymbol{u}_t = \boldsymbol{u}_j
\end{matrix}\right.,
\text{else}
\left\{\begin{matrix}
\boldsymbol{n}_s = \boldsymbol{n}_j, \boldsymbol{u}_s = \boldsymbol{u}_j\\ 
\boldsymbol{n}_t = \boldsymbol{n}_i, \boldsymbol{u}_t = \boldsymbol{u}_i
\end{matrix}\right..
\nonumber
\end{align}

In other words, the source node is chosen to have the smaller angle between its associated $\boldsymbol{u}$ and the vector connecting the two nodes. Finally we can create the fixed coordinate frame $\mathrm{uvw}$ at $\boldsymbol{n}_s$ and compute $\alpha , \beta , \gamma$:
\begin{align}
\small
\left\{\begin{aligned}
\mathrm{u} &= \boldsymbol{u}_s \\
\mathrm{v} &= \mathrm{u} \times \frac{\left(\boldsymbol{n}_t - \boldsymbol{n}_s\right)}{d} \\
\mathrm{w} &=\mathrm{u} \times \mathrm{v}
\end{aligned}\right.,
\left\{\begin{aligned}
\alpha &= \mathrm{u} \cdot \frac{\left(\boldsymbol{n}_t - \boldsymbol{n}_s \right)}{d} \\
\beta &= \mathrm{v} \cdot \boldsymbol{u}_t \\
\gamma &= \arctan \left(\mathrm{w} \cdot \boldsymbol{u}_{t}, \mathrm{u} \cdot \boldsymbol{u}_{t}\right)
\end{aligned}\right..
\label{PFH}
\end{align}

\begin{figure*}[t]
	\centering
	\includegraphics[scale=0.7]{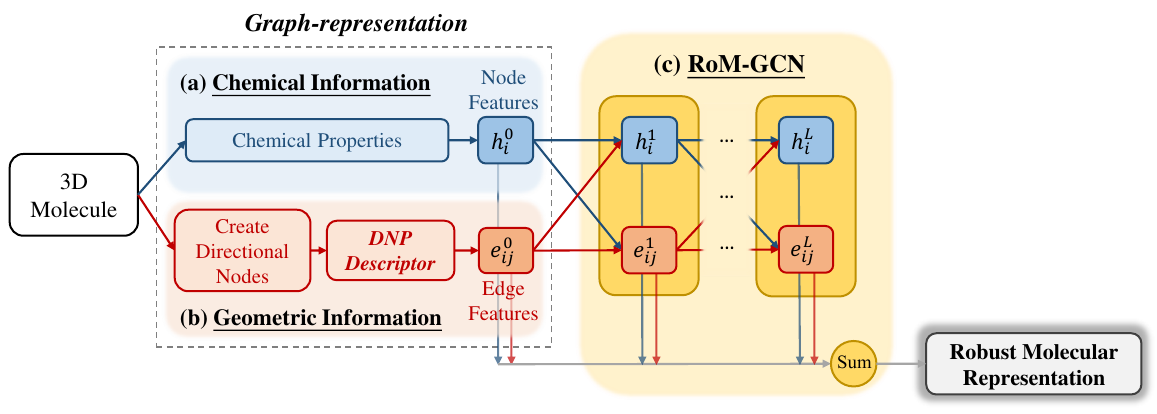}
	\vspace{-5pt}
	\caption{Flowchart of using the RoM-GCN model to learn robust molecular representations.}
	\vspace{-5pt}\label{RoM-GCN_figure}
\end{figure*}

When $\beta = 0$ or $\pi$, we define $\gamma = \pi /2$. The computations are illustrated in in Figure \ref{DNP_figure}(b).

\subsubsection{Properties of DNP Descriptor}
Our proposed DNP descriptor holds the following properites:
\begin{theorem}
The DNP descriptor $D$ of directional node pair $p=(n_i, n_j)$ in graph $G$ has permutational, rotational and translational invariance: For all $n_i, n_j \in G$, $D(p) = f (D(p))$, where $f: \mathbb{R}^3 \rightarrow \mathbb{R}^3$ is a transformation function for permutation, rotation or translation.
\end{theorem}
\begin{proof} (Sketch)
$D$ is permutation invariant since for all $n_i, n_j \in G$, $D(n_i, n_j) = D(n_j, n_i)$. The quadruplets $\left \langle \alpha , \beta , \gamma , d \right \rangle$ in $D$ are either defined for rotational and translational invariant relationships or computed in a fixed coordinate frame $\mathrm{uvw}$ which can be transformed together with $p$. Thus $D$ is rotational and translational invariant.
\end{proof}
\begin{theorem}
The DNP descriptor $D$ of directional node pair $p=(n_i, n_j)$ in graph $G$ is an \textbf{injective} mapping of the geometric relationship between $n_i$ and $n_j$: For all $p_1, p_2$, whenever $D(p_1) = D(p_2)$, then $p_1 = p_2$.
\end{theorem}
\begin{proof} (Sketch)
For any $p$, if it meets the case (2) of the computation of $D$, $D$ will injectively map $p$ to $\left \langle \alpha , \beta , \gamma , d \right \rangle$. If $p$ meets the case (1), the values of $\alpha , \beta , \gamma$ do not conflict with the values in case (2). Thus, $D$ is an injective mapping.
\end{proof}
\begin{corollary}
The DNP descriptor $D$ of directional node pair $p=(n_i, n_j)$ in graph $G$ is robust for all $p$ under Property \ref{Property1} and \ref{Property2}. \label{P_robust}
\end{corollary}

As stated in Corollary \ref{P_robust}, the DNP descriptor is only robust for directional node pairs. However, to represent the structure of a molecule, the geometric relationships among \textit{all} nodes have to be considered. To obtain a robust geometric descriptor for all nodes, we build a GCN in Section \ref{RoM-GCN} with DNP descriptor. The DNP descriptor with the message-passing scheme is an approximation of the robust geometric descriptor for all nodes in $G$, which will be discussed in Section \ref{RoM-GCN}.

\subsubsection{Related Non-robust Descriptors}
Here we cover other existing geometric descriptors that can get the relationships between directional node pairs. None of them are robust under our definitions.

The Point Feature Histograms (PFH) descriptor proposed the computations in Equation (\ref{PFH}) which inspires the idea of our DNP descriptor. Since the PFH descriptor is originally proposed for points associated with surface normal vectors, Equation (\ref{PFH}) can get invalid results when the nodes are associated with arbitrary directional vectors. So that the PFH descriptor is even not a well-defined mapping.

Another popular work~\cite{drost2010model} proposed the Point Pair Feature (PPF) descriptor to describe the geometric relationships between oriented point pairs: For two points $m_1$ and $m_2$ with normals $n_1$ and $n_2$, the feature $\mathbf{F}$ is defined as: $
\mathbf{F}\left(m_1, m_2\right)=\left(\|d\|_{2}, \angle\left({n}_{1}, {d}\right), \angle\left({n}_{2}, {d}\right), \angle\left({n}_{1}, {n}_{2}\right)\right)
$,
where $d=m_2-m_1$, and $\angle\left(a, b\right) \in [0, \pi]$ denotes the angle between vector $a$ and $b$. However, the PPF descriptor fails to be an injective mapping because it will map any pair of chiral structures to the same embedding.

A recent work~\cite{fout2017protein} predicts protein interface using graph representation for protein structures. In the graph representation, each node is an amino acid residue associated with the normal vector of the amide plane. The geometric relationships between two nodes are described using the Euclidean distance $d$ between the nodes and the angle $\theta$ between two normal vectors. This descriptor is not robust because different node pairs may have the same $d$ and $\theta$ values.

\section{Robust Molecular Graph Convolutional Network (RoM-GCN)}\label{RoM-GCN}

In this section, we will design a novel GCN that uses the quadruplets computed from the DNP descriptor proposed in Section \ref{Robust-Descriptor}. Since the DNP descriptor describes the geometric relationships between any given node pair, we consider the quadruplet $\left \langle \alpha , \beta , \gamma , d \right \rangle$ as the input features of the edge that connect the two nodes. As we will discuss later, our GCN learns the robust representations for 3D molecules. So that we name our proposed GCN as the Robust Molecular Graph Convolutional Network (RoM-GCN). Figure \ref{RoM-GCN_figure} shows the flowchart of the RoM-GCN model.

\subsection{Problem Statement}
The problem we want to solve is how to learn the robust graph representation $h_g$ for any given 3D molecule. The representation is considered to be robust if the geometric information that a structure contains is robust under Property \ref{Property1} and \ref{Property2}. Let $G = ( V , E )$ with a set of nodes $V$ and a set of edges $E$ be the graph representation of a 3D molecule. For any node $v_i \in V$, it has node feature $h_i$ and is associated with a directional vector $\boldsymbol{u}_i$. For any edge $\epsilon_{ij} \in E$ connecting node $v_i$ and $v_j$, it has edge feature $e_{ij}$.

The nodes contain the features $h$ representing chemical properties in the 3D molecule related to each node, which is shown in Figure \ref{RoM-GCN_figure}(a). For edge features $e$, we use the DNP descriptor to compute quadruplet $\left \langle \alpha , \beta , \gamma , d \right \rangle$ between two connected nodes as shown in Figure \ref{RoM-GCN_figure}(b). The edge features encode the geometric information of a 3D molecule. With those features, we can use our RoM-GCN to learn a robust representation for different tasks.

\subsection{Framework}
To make our RoM-GCN being injective when updating node and edge features, we use the summation function in all update functions as suggested in~\cite{xu2018how}.

For the node update function of node $v_i$ in the $l$-th convolutional layer, we sum the feature of $v_i$, the features of the neighbor nodes of $v_i$ and the features of the edges connecting $v_i$ and its neighbors:
\begin{align}
h_{i}^{l} = \sigma_{h}^{l} \big( h_{i}^{l-1} + \sum\nolimits_{j \in N(i)}h_{j}^{l-1} + \sum\nolimits_{j \in N(i)}e_{ij}^{l-1}\big), \label{node_update}
\end{align}
where $N(i)$ denotes the neighbors of $v_i$, and the superscript $l$ denotes the items belong to the $l$-th layer. $\sigma_h$ is an injective nonlinear function for node features.

To update the feature of edge $\epsilon_{ij}$, we sum the features of the $v_i$ and $v_j$ together with the feature of $\epsilon_{ij}$:
\begin{align}
e_{ij}^{l} = \sigma_{e}^{l} \big( e_{ij}^{l-1} + h_{i}^{l-1} + h_{j}^{l-1}\big), \label{edge_update}
\end{align}
where $\sigma_e$ is an injective nonlinear function for edge features.

After $L$ convolutional layers, we compute the final representation $h_{g}$ for the whole graph using the following readout function:
\begin{align}
h_{g} = \concat_{k=0}^L \big(\sum\nolimits_{v_i\in V}h_{i}^{k} + \sum\nolimits_{\epsilon_{ij} \in E}e_{ij}^{k}\big), \label{Readout}
\end{align}
where $\parallel$ represents concatenation. 

Comparing with previous GCN-based models proposed for the representation learning of molecules~\cite{kearnes2016molecular,schutt2017schnet,fout2017protein,klicperadirectional}, our RoM-GCN model has the following unique properties: 1. Our model can update node features and edge features in parallel; 2. In the update functions, we use concise summations for injective mappings. 3. In the readout function, we take both node and edge features into consideration and get the representation for a hierarchical structure.

\subsection{Geometric Information in RoM-GCN}
As discussed in Section \ref{Robust-Descriptor}, our DNP descriptor itself is robust for directional node pairs. With the features computed from the DNP descriptor as edge features, we can approximate the geometric relationships between all nodes in a graph using RoM-GCN:

In the first convolutional layer, Equation (\ref{node_update}) aggregates all edge features that connecting to a node $v_i$. Thus, the resulting node feature $h_{i}^{1}$ contains the geometric relationships between $v_i$ and its neighbors, while the geometric information between the neighbors of $v_i$ exists in the updated neighboring node features. In the next layer, Equation (\ref{node_update}) further aggregates all neighboring node features of $v_i$ to get $h_{i}^{2}$. Note that all node features in this layer contain the geometric information computed from the previous layer. Thus, $h_{i}^{2}$ contains the approximate geometric relationships between all nodes that are connected with $v_i$. In \cite{rusu2009fast}, a similar approach was used for the Fast PFH descriptor. As the number of layers getting larger, the node will aggregate the geometric information in a longer range. Finally, using Equation (\ref{Readout}), the graph-level feature will be a robust geometric representation of the whole molecule with a hierarchical structure.

\section{Experiments}
\subsection{Datasets}
\paragraph{Protein datasets.} We adopt the TOUGH-C1 dataset, nucleotide, and heme subsets in~\cite{pu2019deepdrug3d} for the classification of the ligand type of each protein substrate binding pocket to validate our proposed method. TOUGH-C1, nucleotide and heme contain 4164, 3499 and 2542 residues labeled with their ligand types, respectively.

\paragraph{Small molecule datasets.} To evaluate the performance of our model on small molecules, we use the DUD-E decoys datasets~\cite{mysinger2012directory}, which contain active and inactive chemicals for 4 proteins (hivrt, hivpr, cxcr4 or akt1) for the classification task. We use AutoDock Vina~\cite{trott2010autodock} to search for the 3D molecules in the binding pockets. The hivrt, hivpr, cxcr4 and akt1 datasets contain 607, 1395, 122 and 422 active molecules and the same number of inactive molecules, respectively.

\begin{figure}[t]
	\centering
	\includegraphics[scale=0.75]{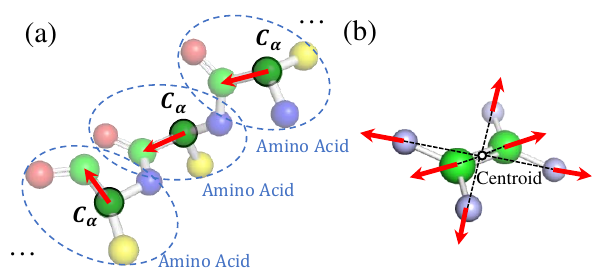}
	\vspace{-5pt}
	\caption{Illustrations of the definition of directional vectors in (a) proteins and (b) small molecules.} 
	\vspace{-5pt}
	\label{Vector}
\end{figure}

\begin{table*}[tbp]
\newcommand{\tabincell}[2]{\begin{tabular}{@{}#1@{}}#2\end{tabular}}
\centering
\resizebox{\linewidth}{!}{
\begin{tabular}{c|c|ccc|c|c|cc|cc|cc}
\toprule
\multicolumn{1}{c|}{\multirow{3}{*}{{\bf Ablation}}} &
\multicolumn{1}{c|}{\multirow{3}{*}{\tabincell{c}{{\bf Geometric} \\ {\bf Descriptor}}}} & \multicolumn{3}{c|}{{\bf Edge Feature}} &
\multicolumn{1}{c|}{\multirow{3}{*}{\tabincell{c}{{\bf \#} \\ {\bf Filters}}}} & \multicolumn{1}{c|}{\multirow{3}{*}{\tabincell{c}{{\bf \#} \\ {\bf Layers}}}} &
\multicolumn{2}{c|}{\multirow{2}{*}{{\bf TOUGH-C1}}} &
\multicolumn{2}{c|}{\multirow{2}{*}{{\bf Nucleotide}}} &
\multicolumn{2}{c}{\multirow{2}{*}{{\bf Heme}}}\\
\cline{3-5}
& & \multirow{2}{*}[-1pt]{\tabincell{c}{{\bf Node} \\ {\bf Update}}} & \multirow{2}{*}[-1pt]{\tabincell{c}{{\bf Edge} \\ {\bf Update}}} & \multirow{2}{*}[-1pt]{{\bf Readout}} & & & & & & & &\\
\cline{8-13}
& & & & & & & \multicolumn{1}{c}{\multirow{1}{*}[-1pt]{{\bf AUC ($\Delta$AUC)}}} & \multirow{1}{*}[-1pt]{{\bf F1 ($\Delta$F1)}} & \multicolumn{1}{c}{\multirow{1}{*}[-1pt]{{\bf AUC ($\Delta$AUC)}}} & \multirow{1}{*}[-1pt]{{\bf F1 ($\Delta$F1)}} & \multicolumn{1}{c}{\multirow{1}{*}[-1pt]{{\bf AUC ($\Delta$AUC)}}} & \multirow{1}{*}[-1pt]{{\bf F1 ($\Delta$F1)}} \\
\midrule
{\bf Reference} & $\alpha , \beta , \gamma , d$ & \checkmark & \checkmark & \checkmark & 128 & 2 & {\bf 0.951} $\ $(0.000) & {\bf 0.775} $\ $(0.000) & {\bf 0.965} $\ $(0.000) & {\bf 0.910} $\ $(0.000) & {\bf 0.983} $\ $(0.000) & {\bf 0.944} $\ $(0.000) \\
\midrule
\multicolumn{1}{c|}{\multirow{2}{*}{\tabincell{c}{\bf Geometric \\ \bf Information}}} & $d$ & \checkmark & \checkmark & \checkmark & 128 & 2 & 0.941 (-0.010) & 0.740 (-0.035) & 0.953 (-0.012) & 0.893 (-0.017) & 0.977 (-0.006) & 0.937 (-0.007) \\
& $\theta , d$ & \checkmark & \checkmark & \checkmark & 128 & 2 & 0.949 (-0.002) & 0.772 (-0.003) & 0.962 (-0.003) & 0.907 (-0.003) & 0.980 (-0.003) & 0.943 (-0.001) \\
\midrule
\multicolumn{1}{c|}{\multirow{4}{*}{\tabincell{c}{\bf Edge \\ \bf Feature}}} & $\alpha , \beta , \gamma , d$ &  &  &  & 128 & 2 & 0.919 (-0.031) & 0.707 (-0.068) & 0.941 (-0.024) & 0.874 (-0.036) & 0.966 (-0.017) & 0.916 (-0.028) \\
& $\alpha , \beta , \gamma , d$ & \checkmark &  &  & 128 & 2 & 0.917 (-0.034) & 0.706 (-0.069) & 0.940 (-0.025) & 0.875 (-0.035) & 0.965 (-0.017) & 0.920 (-0.024) \\
& $\alpha , \beta , \gamma , d$ & \checkmark & \checkmark &  & 128 & 2 & 0.949 (-0.002) & 0.772 (-0.003) & {\bf 0.965} ($\ $0.000) & {\bf 0.910} ($\ $0.000) & 0.981 (-0.002) & 0.943 (-0.001) \\
& $\alpha , \beta , \gamma , d$ & \checkmark &  & \checkmark & 128 & 2 & 0.949 (-0.002) & 0.762 (-0.013) & 0.960 (-0.004) & 0.902 (-0.008) & 0.976 (-0.006) & 0.933 (-0.011) \\
\midrule
\multicolumn{1}{c|}{\multirow{3}{*}{{\bf Width}}} & $\alpha , \beta , \gamma , d$ & \checkmark & \checkmark & \checkmark & 16 & 2 & 0.927 (-0.024) & 0.716 (-0.059) & 0.939 (-0.026) & 0.872 (-0.038) & 0.969 (-0.014) & 0.918 (-0.026) \\
& $\alpha , \beta , \gamma , d$ & \checkmark & \checkmark & \checkmark & 32 & 2 & 0.935 (-0.016) & 0.733 (-0.042) & 0.951 (-0.014) & 0.889 (-0.021) & 0.978 (-0.005) & 0.935 (-0.009) \\
& $\alpha , \beta , \gamma , d$ & \checkmark & \checkmark & \checkmark & 64 & 2 & 0.945 (-0.006) & 0.769 (-0.006) & 0.963 (-0.002) & 0.908 (-0.002) & 0.980 (-0.003) & 0.938 (-0.006) \\
\midrule
\multicolumn{1}{c|}{\multirow{4}{*}{{\bf Depth}}} & $\alpha , \beta , \gamma , d$ & \checkmark & \checkmark & \checkmark & 128 & 0 & 0.947 (-0.004) & 0.758 (-0.017) & 0.958 (-0.007) & 0.896 (-0.014) & 0.977 (-0.006) & 0.933 (-0.011) \\
& $\alpha , \beta , \gamma , d$ & \checkmark & \checkmark & \checkmark & 128 & 1 & {\bf 0.952} ($\ $0.001) & {\bf 0.775} ($\ $0.000) & {\bf 0.965} ($\ $0.000) & {\bf 0.911} ($\ $0.001) & 0.981 (-0.002) & 0.942 (-0.002) \\
& $\alpha , \beta , \gamma , d$ & \checkmark & \checkmark & \checkmark & 128 & 3 & 0.946 (-0.005) & 0.759 (-0.016) & 0.964 (-0.001) & 0.908 (-0.002) & 0.981 (-0.002) & 0.941 (-0.003) \\
& $\alpha , \beta , \gamma , d$ & \checkmark & \checkmark & \checkmark & 128 & 4 & 0.949 (-0.002) & 0.768 (-0.007) & 0.963 (-0.002) & {\bf 0.910} ($\ $0.000) & 0.981 (-0.002) & 0.943 (-0.001) \\
\bottomrule
\end{tabular}
}
\vspace{-5pt}
\caption{Results of ablation study of RoM-GCN on protein datasets. We compare the reference model to ablated variants. We use the same configurations and hyper-parameters except those for ablations on all models. The AUC and F1 scores are reported. $\Delta$AUC and $\Delta$F1 denote the difference with respect to the reference model. We highlight the results of the reference model and the results higher than the reference.}
\vspace{-5pt}
\label{Ablation}
\end{table*}

\begin{table}[tbp]
\newcommand{\tabincell}[2]{\begin{tabular}{@{}#1@{}}#2\end{tabular}}
\centering
\resizebox{\linewidth}{!}{
\begin{tabular}{c|cc|cc|cc}
\toprule
\multicolumn{1}{c|}{\multirow{2}{*}{{\bf Dataset}}} &
\multicolumn{2}{c|}{\multirow{1}{*}{{\bf PointNet}}} &
\multicolumn{2}{c|}{\multirow{1}{*}{{\bf 3DCNN}}} &
\multicolumn{2}{c}{\multirow{1}{*}{{\bf RoM-GCN}}}\\
\cline{2-7}
& \multicolumn{1}{c}{\multirow{1}{*}[-2pt]{{\bf AUC (std)}}} & \multirow{1}{*}[-2pt]{{\bf F1 (std)}} & \multicolumn{1}{c}{\multirow{1}{*}[-2pt]{{\bf AUC (std)}}} & \multirow{1}{*}[-2pt]{{\bf F1 (std)}} & \multicolumn{1}{c}{\multirow{1}{*}[-2pt]{{\bf AUC (std)}}} & \multirow{1}{*}[-2pt]{{\bf F1 (std)}} \\
\midrule
{\bf TOUGH-C1} & 0.852 (0.025) & 0.545 (0.033) & 0.935 (0.020) & {\bf 0.788 (0.020)} & {\bf 0.951 (0.018)} & 0.775 (0.046) \\
{\bf Nucleotide} & 0.746 (0.023) & 0.651 (0.026) & 0.913 (0.012) & 0.819 (0.014) & {\bf 0.965 (0.009)} & {\bf 0.910 (0.015)} \\
{\bf Heme} & 0.822 (0.029) & 0.600 (0.052) & 0.947 (0.018) & 0.811 (0.038) & {\bf 0.983 (0.007)} & {\bf 0.944 (0.019)} \\
\midrule
{\bf hivrt} & 0.758 (0.034) & 0.698 (0.034) & 0.915 (0.033) & 0.840 (0.040) & {\bf 0.984 (0.009)} & {\bf 0.949 (0.016)}\\
{\bf hivpr} & 0.919 (0.030) & 0.840 (0.033) & 0.984 (0.020) & 0.945 (0.022) & {\bf 0.997 (0.003)} & {\bf 0.991 (0.005)}\\
{\bf cxcr4} & 0.784 (0.088) & 0.745 (0.062) & 0.960 (0.049) & 0.917 (0.048) & {\bf 0.990 (0.021)} & {\bf 0.982 (0.022)}\\
{\bf akt1}  & 0.894 (0.028) & 0.799 (0.026) & 0.974 (0.023) & 0.923 (0.028) & {\bf 0.996 (0.008)} & {\bf 0.980 (0.015)}\\
\bottomrule
\end{tabular}
}
\vspace{-5pt}
\caption{Results of RoM-GCN and baseline models on protein and small molecule datasets.}
\vspace{-5pt}
\label{Baseline}
\end{table}

\subsection{Models and Configurations}
To test our RoM-GCN model, we use PointNet~\cite{qi2017pointnet} and the 3DCNN in~\cite{pu2019deepdrug3d} as baselines. The input graphs of RoM-GCN are created as follows: In the protein datasets, we treat each amino acid residue as a node. The position of alpha carbon $C_{\alpha}$ in each residue is used to define edges: Two residues are connected with an edge if the distance between their $C_{\alpha}$ is smaller than 15 $\mathrm{\AA}$. The directional vector of each node is defined as the direction from $C_{\alpha}$ to the carboxyl carbon in the same residue, which is shown in Figure \ref{Vector}(a). In small molecules datasets, we consider each heavy atom as a node and each bond as an edge. To define directional vectors, we first find the centroid for each molecule, then the direction is from the global centroid to each atom as shown in Figure \ref{Vector}(b). In this way, the directional vectors reserve global geometric information. To create the input of 3DCNN, we use a 33x33x33 input voxel grid contains only the alpha carbons in each protein. For the input of PointNet, we consider each node in the graphs for RoM-GCN as a point to build a point cloud for each molecule. Moreover, for 3DCNN and PointNet, in order to standardize the input orientation, we apply the same principle component based rotation procedure as in~\cite{pu2019deepdrug3d}.

In RoM-GCN, we use 2-layer MLP to approximate all injective non-linear functions. To make the comparison fair, we fix the number of layers as 2 and the number of filters as 128 in all models. The classifiers are 2-layer MLP classifiers with the Rectified Linear Units. For the imbalanced datasets, we use the weighted cross-entropy loss function. Batch normalization is applied to every layer in MLP. All models are trained using the Adam optimizer with an initial learning rate between $1e^{-3}$ and $1e^{-4}$ and the learning rate decays $50\%$ every 50 epochs. For all experiments, we perform 10-fold cross-validation and repeat the experiments 5 times for each dataset and each model. We choose the epoch with the best testing accuracy averaged over all folds to report the results, which are the AUC (Area under the Curve of ROC) and F1 scores.

\subsection{Results}
\subsubsection{Ablation Study on Protein Datasets}
We conduct an ablation study of our RoM-GCN model on the protein binding pocket datasets. In the study, we compare different geometric descriptors used for the edge features, different ways to incorporate the edge features and different parameters (number of filters and layers) of the model. A reference RoM-GCN model is compared with ablated variants. The goal of the ablation study is to investigate the contribution of each component in our model. In Table~\ref{Ablation} we list the configurations of all RoM-GCN variants and report the results of the ablation study.

\paragraph{Effect of geometric descriptor.} There are three different geometric descriptors to be compared: 1. Only Euclidean distance $d$ between nodes; 2. $d$ and the angle $\theta$ between the normal vectors of the amide plane~\cite{fout2017protein}; 3. Our DNP descriptor using $\alpha , \beta , \gamma, d$. The results in Table~\ref{Ablation} show that our DNP descriptor outperforms the descriptor that only has distance information. The model adopting the descriptor used in~\cite{fout2017protein} can get comparable performance with the reference model, because, in protein structures, the information of $d$ and $\theta$ is sufficient to determine the unique conformation due to the constraints of peptide bonds. Thus, this descriptor is robust merely for protein structures, but cannot be applied to small molecules where theta cannot be defined. The improved performance of using robust descriptors implies that the geometric information plays a critical role in the molecular representation learning task.

\paragraph{Effect of ways to incorporate edge features.} We test different ways to incorporate the edge features in our RoM-GCN model, which include ignoring edge features, adopting edge features in the node updates, having the edge updates, and using edge features in the readout function. Results in Table~\ref{Ablation} show that ignoring edge features or only adopting edge features in the node updates get the worst performance. While by adding the edge update function or using the edge features in the readout function, the performance can be improved compared to the model without any edge information. Thus, it is necessary to correctly use edge features for better results. Our reference model gets the overall best performance by using all three components to incorporate the edge features.

\paragraph{Effect of parameters.} We evaluate the effects of the number of filters and layers of our RoM-GCN model. The results in Table~\ref{Ablation} show that a larger number of filters improves network capacity, which is necessary for learning better molecular representations. However, simply increasing the number of layers does not guarantee a better performance without careful design of the network.

\subsubsection{Comparison with Baselines}
We use the reference RoM-GCN in our ablation study to compare with baseline models. In Table~\ref{Baseline} we show the comparison results of all models on the protein and small molecule datasets. On all datasets, the performance of PointNet has large gaps comparing with 3DCNN and our RoM-GCN. This implies that although the point cloud representation can fully capture the geometric information of a 3D molecular structure using xyz coordinates, it does not exhibit good generalization power for molecules. For 3DCNN, it outperforms PointNet with the voxelized input generated from 3D molecular structures and can get the highest F1 score on the TOUGH-C1 dataset. However, neither of the 3D molecular representations of PointNet nor 3DCNN captures robust and rich geometric information. In contrast, our RoM-GCN uses the DNP descriptor to extract robust geometric information from molecular graphs with directional nodes. Moreover, the rich directional information shown in Figure \ref{Vector} can be explicitly embedded and updated using our novel aggregation algorithms in RoM-GCN. As a result, RoM-GCN overall outperforms the baselines on all seven datasets.

\section{Conclusion}
In this work, we successfully developed a universal method to learn the representations with robust geometric information for different types of 3D molecules. Particularly, we proposed the DNP descriptor to obtain the robust geometric information in 3D molecular structures. To exploit the geometric and chemical features captured in the DNP descriptor, we propose the RoM-GCN model, which is capable to take the advantage of both node and edge features and learn robust representations at the molecular level. Experiments on protein and small molecule datasets demonstrate the superiority of the DNP descriptor and the RoM-GCN model. As a future direction, it would be interesting to develop better ways to define directional vectors. 

\bibliographystyle{named}
\bibliography{main}

\end{document}